\documentclass[10pt, a4paper, twocolumn]{article}

\usepackage[utf8]{inputenc}
\usepackage[T1]{fontenc}
\usepackage{amsmath, amssymb, amsthm, amsfonts}
\usepackage{geometry}
\usepackage{microtype}
\usepackage{enumitem}
\usepackage{hyperref}
\usepackage{cite}
\usepackage{mathtools}

\hypersetup{
    colorlinks=true,
    linkcolor=blue,
    citecolor=blue,
    urlcolor=blue,
    pdfauthor={Przemysław Spyra},
    pdftitle={Unconditional Convergence in Banach Spaces}
}

\newtheorem{theorem}{Theorem}[section]
\newtheorem{lemma}[theorem]{Lemma}
\newtheorem{proposition}[theorem]{Proposition}
\newtheorem{corollary}[theorem]{Corollary}
\newtheorem{definition}[theorem]{Definition}
\newtheorem{example}[theorem]{Example}
\newtheorem{remark}[theorem]{Remark}

\newcommand{\R}{\mathbb{R}}
\newcommand{\C}{\mathbb{C}}
\newcommand{\N}{\mathbb{N}}
\newcommand{\K}{\mathbb{K}}
\newcommand{\X}{\mathcal{X}}
\newcommand{\Hs}{\mathcal{H}}
\newcommand{\norm}[1]{\|#1\|}
\newcommand{\abs}[1]{|#1|}
\newcommand{\ip}[2]{\langle #1, #2 \rangle}

\title{
    \sffamily\bfseries 
    Algorithmic Stability in Infinite Dimensions:\\[0.5em] 
    \Large Characterizing Unconditional Convergence in Banach Spaces
}

\author{
    \textbf{Przemysław Spyra} \\[0.5em]
    \small Faculty of Applied Mathematics \\
    \small AGH University of Science and Technology \\
    \small Krakow, Poland \\[0.3em]
    \small \texttt{przspyra@student.agh.edu.pl}
}

\date{}

\begin{document}

\maketitle

\begin{abstract}
\noindent 
The distinction between conditional, unconditional, and absolute convergence in infinite-dimensional spaces has fundamental implications for computational algorithms. While these concepts coincide in finite dimensions, the Dvoretzky-Rogers theorem establishes their strict separation in general Banach spaces. We present a comprehensive characterization theorem unifying seven equivalent conditions for unconditional convergence: permutation invariance, net convergence, subseries tests, sign stability, bounded multiplier properties, and weak uniform convergence. These theoretical results directly inform algorithmic stability analysis, governing permutation invariance in gradient accumulation for Stochastic Gradient Descent and justifying coefficient thresholding in frame-based signal processing. Our work bridges classical functional analysis with contemporary computational practice, providing rigorous foundations for order-independent and numerically robust summation processes.

\vspace{0.4em}
\end{abstract}

\section{Introduction}

\subsection{Motivation and Context}

The convergence behavior of infinite series lies at the heart of mathematical analysis and its applications. A classical result, the Riemann Rearrangement Theorem, reveals a striking phenomenon: a conditionally convergent series $\sum_{n=1}^\infty a_n$ can be rearranged to converge to \emph{any} prescribed value or diverge entirely. By contrast, absolutely convergent series exhibit \emph{permutation invariance}, with sums independent of summation order.

This distinction extends far beyond pure mathematics. In machine learning, neural network training via Stochastic Gradient Descent involves accumulating gradient updates $\sum_{i=1}^N g_i$ across mini-batches. \\ The mathematical properties of this summation determine whether results depend on data ordering, a critical concern for reproducibility and distributed training. Similarly, signal processing representations $f = \sum_{n=1}^\infty c_n \phi_n$ using frame expansions require stability under coefficient modifications such as thresholding and quantization. Conditional convergence renders such operations fundamentally unpredictable.

\subsection{The Infinite-Dimensional Challenge}

In finite-dimensional spaces $\R^n$ or $\C^n$, absolute convergence (convergence of $\sum \norm{x_n}$) and unconditional convergence (convergence under all permutations) are equivalent. This equivalence fails dramatically in infinite dimensions. The celebrated Dvoretzky-Rogers theorem (1950) establishes that \emph{every} infinite-dimensional Banach space admits series converging unconditionally but not absolutely~\cite{dvoretzky1950}. This creates a strict hierarchy:
\[
\text{Absolute} \subsetneq \text{Unconditional} \subsetneq \text{Conditional}.
\]
Understanding this hierarchy is essential for both theoretical analysis and practical algorithm design.

\subsection{Contributions}

This paper makes three primary contributions:

\textbf{1. Unified Characterization.} We present Theorem~\ref{thm:main}, synthesizing seven equivalent formulations of unconditional convergence spanning topological (net convergence), algebraic (subseries and sign stability), and functional-analytic (bounded multipliers, weak uniform convergence) perspectives.

\textbf{2. Rigorous Proofs.} We provide complete, self-contained proofs organized in logical cycles, clarifying interdependencies between conditions with explicit estimates and constructive arguments.

\textbf{3. Algorithmic Applications.} We connect abstract results to concrete computational problems: establishing when gradient accumulation remains ordering-invariant in machine learning, and justifying coefficient manipulation in signal processing under unconditional bases.

\subsection{Related Work}

The theory of unconditional convergence has deep roots in functional analysis. Orlicz (1933) proved that unconditional convergence in Hilbert spaces implies square-summability~\cite{orlicz1933}. The Dvoretzky-Rogers theorem~\cite{dvoretzky1950} showed this cannot be strengthened to absolute convergence. Standard references include Heil~\cite{heil2011}, Megginson~\cite{megginson1998}, and Rudin~\cite{rudin1991}.

Applications to optimization have been explored in coordinate descent~\cite{wright2015} and randomized algorithms~\cite{richtarik2014}. Our work emphasizes mathematical prerequisites for algorithmic stability in both deterministic and stochastic settings.

\subsection{Organization}

Section~\ref{sec:prelim} establishes notation and definitions. Section~\ref{sec:finite} proves finite-dimensional equivalence. Section~\ref{sec:infinite} presents the infinite-dimensional gap via Dvoretzky-Rogers and Orlicz theorems. Section~\ref{sec:main} contains our main characterization theorem with complete proofs. Section~\ref{sec:applications} develops applications to machine learning and signal processing. Section~\ref{sec:conclusion} concludes with open problems.

\section{Preliminaries}
\label{sec:prelim}

Let $\K \in \{\R, \C\}$ denote the scalar field. We assume familiarity with normed spaces and basic functional analysis~\cite{rudin1991}.

\begin{definition}[Banach and Hilbert Spaces]
A \emph{Banach space} is a complete normed vector space $(\X, \norm{\cdot})$. A \emph{Hilbert space} $\Hs$ is a Banach space whose norm is induced by an inner product: $\norm{x}^2 = \ip{x}{x}$.
\end{definition}

\begin{definition}[Series Convergence]
\label{def:convergence}
Let $\{x_n\}_{n=1}^\infty \subset \X$ be a sequence in a Banach space.

\begin{enumerate}[label=(\alph*), leftmargin=2em, itemsep=0.3em]
\item The series $\sum_{n=1}^\infty x_n$ \emph{converges} to $x \in \X$ if
\[
\lim_{N \to \infty} \norm{x - \sum_{n=1}^N x_n} = 0.
\]

\item The series is \emph{absolutely convergent} if $\sum_{n=1}^\infty \norm{x_n} < \infty$.

\item The series is \emph{unconditionally convergent} if for every bijection $\sigma : \N \to \N$, the rearranged series $\sum_{n=1}^\infty x_{\sigma(n)}$ converges.
\end{enumerate}
\end{definition}

\begin{remark}
Absolute convergence trivially implies convergence by completeness. Unconditional convergence also implies convergence (take $\sigma = \text{identity}$). The converse implications are our focus.
\end{remark}

To characterize unconditional convergence topologically, we introduce net convergence.

\begin{definition}[Directed Sets and Nets]
A set $I$ with relation $\leq$ is \emph{directed} if: (i)~$\leq$ is reflexive and transitive; (ii)~$\forall i, j \in I$, $\exists k \in I$ such that $i \leq k$ and $j \leq k$. A \emph{net} in $\X$ is a map $\phi : I \to \X$ indexed by a directed set~$I$.
\end{definition}

\begin{definition}[Convergence via Nets]
Let $\mathcal{F}$ denote the collection of finite subsets of $\N$, directed by inclusion $\subseteq$. We say $\sum x_n$ converges unconditionally to $x$ with respect to $\mathcal{F}$ (written $x = \lim_{F \in \mathcal{F}} \sum_{n \in F} x_n$) if
\[
\forall \varepsilon > 0, \; \exists F_0 \in \mathcal{F}, \; \forall F \supseteq F_0 : \norm{x - \sum_{n \in F} x_n} < \varepsilon.
\]
\end{definition}

\begin{remark}
If the net limit exists, then the ordinary series converges. Indeed, choose $F_0$ for given $\varepsilon$, let $N_0 = \max F_0$, and note $\{1, \ldots, N\} \supseteq F_0$ for all $N \geq N_0$.
\end{remark}

\section{Finite-Dimensional Equivalence}
\label{sec:finite}

We establish that absolute and unconditional convergence coincide in finite dimensions, intuition that fails to generalize.

\begin{lemma}[Complex Scalars]
\label{lem:complex}
Let $\{c_n\}_{n=1}^\infty \subset \C$. Then $\sum c_n$ converges absolutely if and only if it converges unconditionally.
\end{lemma}

\begin{proof}
($\Rightarrow$) Standard: if $\sum \abs{c_n} < \infty$, then $\sum c_{\sigma(n)}$ is Cauchy for any permutation $\sigma$, hence convergent.

($\Leftarrow$) Suppose $\sum c_n$ converges unconditionally but not absolutely. Write $c_n = a_n + ib_n$ with $a_n, b_n \in \R$. If $\sum \abs{c_n} = \infty$, then $\sum \abs{a_n} = \infty$ or $\sum \abs{b_n} = \infty$. Without loss of generality, assume $\sum \abs{a_n} = \infty$.

Since $\sum a_n$ converges (by unconditional convergence of $\sum c_n$) yet $\sum \abs{a_n} = \infty$, the series $\sum a_n$ is conditionally convergent. By the Riemann Rearrangement Theorem, there exists a permutation $\tau$ such that $\sum a_{\tau(n)}$ diverges. Defining $\sigma(n) = \tau(n)$ yields a permutation where $\sum c_{\sigma(n)}$ diverges, contradicting unconditional convergence.
\end{proof}

\begin{proposition}[Finite-Dimensional Spaces]
\label{prop:finite}
Let $\{f_n\}_{n=1}^\infty$ be a sequence in $\X = \R^d$ (or $\C^d$). Then $\sum f_n$ converges absolutely if and only if it converges unconditionally.
\end{proposition}

\begin{proof}
Write $f_n = \sum_{j=1}^d c_n^{(j)} e_j$ where $\{e_1, \ldots, e_d\}$ is the standard basis. Then
\[
\sum_{n=1}^\infty \norm{f_n} < \infty \iff \sum_{n=1}^\infty \abs{c_n^{(j)}} < \infty \; \forall j = 1, \ldots, d.
\]
The forward direction uses $\abs{c_n^{(j)}} \leq \norm{f_n}$; the reverse uses norm equivalence in finite dimensions. By Lemma~\ref{lem:complex}, coordinate-wise absolute convergence is equivalent to coordinate-wise unconditional convergence. Since convergence in $\R^d$ is equivalent to coordinate-wise convergence, the result follows.
\end{proof}

\section{The Infinite-Dimensional Gap}
\label{sec:infinite}

The equivalence between absolute and unconditional convergence is purely finite-dimensional. In infinite dimensions, these concepts diverge.

\begin{lemma}
\label{lem:abs_implies_uncond}
In any Banach space $\X$, absolute convergence implies unconditional convergence.
\end{lemma}

\begin{proof}
Suppose $\sum \norm{x_n} < \infty$. For any permutation $\sigma$, we have $\sum \norm{x_{\sigma(n)}} = \sum \norm{x_n} < \infty$. Since $\X$ is complete, the rearranged series $\sum x_{\sigma(n)}$ converges.
\end{proof}

The converse fails spectacularly:

\begin{theorem}[Dvoretzky-Rogers, 1950~\cite{dvoretzky1950}]
\label{thm:DR}
Every infinite-dimensional Banach space $\X$ contains a series $\sum x_n$ that converges unconditionally but not absolutely.
\end{theorem}

\begin{remark}
The original proof constructs $x_n$ with $\norm{x_n} = 1$ and $\{c_n\} \in \ell^2 \setminus \ell^1$ such that $\sum c_n x_n$ converges unconditionally. This uses Dvoretzky's theorem on almost-spherical sections of convex bodies~\cite{dvoretzky1950}.
\end{remark}

\begin{example}[Hilbert Space Illustration]
\label{ex:hilbert}
Let $\Hs = \ell^2$ with orthonormal basis $\{e_n\}_{n=1}^\infty$. Define $x_n = \frac{1}{n} e_n$.

\textbf{Claim 1:} $\sum x_n$ converges unconditionally.

\textit{Proof:} For any permutation $\sigma$, orthogonality gives
\[
\norm{\sum_{k=1}^N x_{\sigma(k)}}^2 = \sum_{k=1}^N \frac{1}{\sigma(k)^2} \leq \sum_{k=1}^\infty \frac{1}{k^2} < \infty.
\]
Hence partial sums are Cauchy, so convergent. \qed

\textbf{Claim 2:} $\sum x_n$ does not converge absolutely.

\textit{Proof:} $\sum \norm{x_n} = \sum \frac{1}{n} = \infty$. \qed
\end{example}

For Hilbert spaces, we have a partial converse to Lemma~\ref{lem:abs_implies_uncond}:

\begin{theorem}[Orlicz, 1933~\cite{orlicz1933}]
\label{thm:orlicz}
If $\sum x_n$ converges unconditionally in a Hilbert space $\Hs$, then $\sum \norm{x_n}^2 < \infty$.
\end{theorem}

\begin{remark}
Orlicz's theorem shows unconditional convergence in Hilbert spaces lies between $\ell^1$ and $\ell^2$ summability. The gap between $\ell^2$ and $\ell^1$ (e.g., $\sum 1/n$) is exactly where unconditional-but-not-absolute series live.
\end{remark}

\section{The Characterization Theorem}
\label{sec:main}

We now present our central result: seven equivalent characterizations of unconditional convergence unifying topological, algebraic, and functional-analytic perspectives.

\begin{theorem}[Characterization of Unconditional Convergence]
\label{thm:main}
Let $\{x_n\}_{n=1}^\infty$ be a sequence in a Banach space $\X$. The following are equivalent:

\begin{enumerate}[label=(\roman*), leftmargin=2.5em, itemsep=0.4em]
\item $\sum x_n$ converges unconditionally.

\item The net limit $\lim_{F \in \mathcal{F}} \sum_{n \in F} x_n$ exists.

\item Cauchy condition for nets: 
$\forall \varepsilon > 0$, $\exists N \in \N$, $\forall F \in \mathcal{F}$ with $\min(F) > N$: $\norm{\sum_{n \in F} x_n} < \varepsilon$.

\item Every subseries converges: for all strictly increasing $\{n_k\}_{k=1}^\infty \subset \N$, the series $\sum_{k=1}^\infty x_{n_k}$ converges.

\item Every signed series converges: for all $\{\varepsilon_n\}_{n=1}^\infty \subset \{-1, +1\}$, the series $\sum \varepsilon_n x_n$ converges.

\item Bounded multiplier property: For all bounded sequences $\{\lambda_n\}_{n=1}^\infty \in \ell^\infty$, the series $\sum \lambda_n x_n$ converges.

\item Weak uniform convergence:
\[
\lim_{N \to \infty} \sup_{\substack{x^* \in \X^* \\ \norm{x^*} \leq 1}} \sum_{n=N}^\infty \abs{\ip{x_n}{x^*}} = 0.
\]
\end{enumerate}
\end{theorem}

\begin{proof}
We prove the cycle of implications: $(i) \Rightarrow (ii) \Rightarrow (iii) \Rightarrow (i)$, and separately $(iii) \Leftrightarrow (iv) \Leftrightarrow (v)$, $(v) \Leftrightarrow (vi)$, and $(iii) \Rightarrow (vii) \Rightarrow (vi)$.

\textbf{$(i) \Rightarrow (ii)$:} Suppose $\sum x_n$ converges unconditionally to $x$, but the net limit does not exist. Then there exists $\varepsilon > 0$ such that for every finite $F_0 \subset \N$, there exists $F \supseteq F_0$ (finite) with
\[
\norm{x - \sum_{n \in F} x_n} \geq \varepsilon.
\]
Since $\sum x_n = x$, there exists $M_1$ such that $\norm{x - \sum_{n=1}^{M_1} x_n} < \varepsilon/2$. Set $F_1 = \{1, \ldots, M_1\}$. By assumption, there exists finite $G_1 \supseteq F_1$ with $\norm{x - \sum_{n \in G_1} x_n} \geq \varepsilon$. Let $M_2 = \max G_1 + 1$ and set $F_2 = \{1, \ldots, M_2\}$. Then $\norm{x - \sum_{n \in F_2} x_n} < \varepsilon/2$ (for large $M_2$).

Continuing, we construct $F_1 \subset G_1 \subset F_2 \subset G_2 \subset \cdots$ with
\[
\norm{x - \sum_{F_k} x_n} < \varepsilon/2, \quad \norm{x - \sum_{G_k} x_n} \geq \varepsilon.
\]
Thus
\[
\norm{\sum_{G_k \setminus F_k} x_n} \geq \norm{x - \sum_{G_k} x_n} - \norm{x - \sum_{F_k} x_n} \geq \varepsilon/2.
\]
Define a permutation $\sigma$ by listing elements of $F_1$, then $G_1 \setminus F_1$, then $F_2 \setminus G_1$, then $G_2 \setminus F_2$, etc. For this permutation, partial sums fail to be Cauchy (consecutive blocks of size $|G_k \setminus F_k|$ have norm $\geq \varepsilon/2$), contradicting unconditional convergence.

\textbf{$(ii) \Rightarrow (iii)$:} Suppose $x = \lim_{F \in \mathcal{F}} \sum_{n \in F} x_n$ exists. Fix $\varepsilon > 0$. By definition, there exists finite $F_0$ such that for all $F \supseteq F_0$:
\[
\norm{x - \sum_{n \in F} x_n} < \varepsilon/2.
\]
Let $N = \max F_0$. For any finite $G$ with $\min(G) > N$, we have $G \cap F_0 = \emptyset$, so
\[
\norm{\sum_{n \in G} x_n} = \norm{\left(x - \sum_{F_0} x_n\right) - \left(x - \sum_{F_0 \cup G} x_n\right)} \leq \varepsilon/2 + \varepsilon/2 = \varepsilon.
\]

\textbf{$(iii) \Rightarrow (i)$:} Let $\sigma : \N \to \N$ be any permutation. We show $\sum x_{\sigma(n)}$ is Cauchy. Fix $\varepsilon > 0$ and let $N$ satisfy (iii). Choose $M$ large enough that $\{\sigma(1), \ldots, \sigma(M)\} \supseteq \{1, \ldots, N\}$.

For $L > K \geq M$, define $F = \{\sigma(K+1), \ldots, \sigma(L)\}$. Since $F \cap \{1, \ldots, N\} = \emptyset$, we have $\min(F) > N$. By (iii):
\[
\norm{\sum_{n=K+1}^L x_{\sigma(n)}} = \norm{\sum_{n \in F} x_n} < \varepsilon.
\]
Thus $\sum x_{\sigma(n)}$ is Cauchy, hence convergent.

\textbf{$(iii) \Rightarrow (iv)$:} Let $0 < n_1 < n_2 < \cdots$ be strictly increasing. Fix $\varepsilon > 0$ and let $N$ satisfy (iii). Choose $j$ such that $n_j > N$. For $\ell > k \geq j$:
\[
F := \{n_{k+1}, \ldots, n_\ell\} \text{ satisfies } \min(F) \geq n_j > N.
\]
By (iii), $\norm{\sum_{i=k+1}^\ell x_{n_i}} < \varepsilon$. Thus $\sum x_{n_k}$ is Cauchy, hence convergent.

\textbf{$(iv) \Rightarrow (iii)$:} Suppose (iii) fails. Then there exists $\varepsilon > 0$ such that for all $N$, there exists finite $F_N$ with $\min(F_N) > N$ and $\norm{\sum_{F_N} x_n} \geq \varepsilon$.

Construct disjoint finite sets: $G_1 = F_1$, $N_1 = \max G_1$, $G_2 = F_{N_1+1}$, $N_2 = \max G_2$, etc. Then $\max G_k < \min G_{k+1}$ and $\norm{\sum_{G_k} x_n} \geq \varepsilon$ for all $k$.

Let $\{n_1, n_2, \ldots\}$ be the enumeration of $\bigcup_{k=1}^\infty G_k$ in increasing order. Then $\sum x_{n_j}$ is not Cauchy (blocks $G_k$ contribute $\geq \varepsilon$), contradicting (iv).

\textbf{$(iv) \Leftrightarrow (v)$:}

\textit{$(iv) \Rightarrow (v)$:} Given $\{\varepsilon_n\} \subset \{-1, +1\}$, define
\[
I_+ = \{n : \varepsilon_n = +1\}, \quad I_- = \{n : \varepsilon_n = -1\}.
\]
By (iv), both $\sum_{n \in I_+} x_n$ and $\sum_{n \in I_-} x_n$ converge. Hence
\[
\sum \varepsilon_n x_n = \sum_{I_+} x_n - \sum_{I_-} x_n
\]
converges.

\textit{$(v) \Rightarrow (iv)$:} Given increasing $\{n_k\}$, define
\[
\varepsilon_n = \begin{cases}
+1 & \text{if } n = n_k \text{ for some } k, \\
-1 & \text{otherwise}.
\end{cases}
\]
By (v), both $\sum \varepsilon_n x_n$ and $\sum (+1) x_n$ converge. Since
\[
\sum x_{n_k} = \frac{1}{2}\left(\sum x_n + \sum \varepsilon_n x_n\right),
\]
the subseries converges.

\textbf{$(v) \Rightarrow (vi)$:} Trivial: $\{\varepsilon_n\} \subset \ell^\infty$ with $|\varepsilon_n| = 1$.

\textbf{$(vi) \Rightarrow (v)$:} Trivial: signs are a special case of bounded multipliers.

\textbf{$(iii) \Rightarrow (vii)$:} Fix $\varepsilon > 0$ and let $N$ satisfy (iii). For $K \geq N$ and $x^* \in \X^*$ with $\norm{x^*} \leq 1$, consider $\sum_{n=K}^L |\ip{x_n}{x^*}|$.

For each $n$, choose $\theta_n \in \K$ with $|\theta_n| = 1$ and $\theta_n \ip{x_n}{x^*} = |\ip{x_n}{x^*}|$ (phase alignment). Then
\begin{align*}
\sum_{n=K}^L |\ip{x_n}{x^*}| &= \sum_{n=K}^L \theta_n \ip{x_n}{x^*} = \text{Re}\, \ip{\sum_{n=K}^L \theta_n x_n}{x^*} \\
&\leq \left|\ip{\sum_{n=K}^L \theta_n x_n}{x^*}\right| \leq \norm{x^*} \cdot \norm{\sum_{n=K}^L \theta_n x_n}.
\end{align*}
Set $F = \{K, K+1, \ldots, L\}$. Since $\min(F) = K > N$ (for $K > N$), condition (iii) gives $\norm{\sum_{n \in F} \theta_n x_n} < \varepsilon$. Thus
\[
\sum_{n=K}^L |\ip{x_n}{x^*}| < \varepsilon \cdot \norm{x^*} \leq \varepsilon.
\]
Taking $L \to \infty$ and then $\sup$ over $\norm{x^*} \leq 1$:
\[
\sup_{\norm{x^*} \leq 1} \sum_{n=K}^\infty |\ip{x_n}{x^*}| \leq \varepsilon.
\]
Since $\varepsilon$ was arbitrary and $K$ can be chosen arbitrarily large, (vii) holds.

\textbf{$(vii) \Rightarrow (vi)$:} Assume (vii) and let $\{\lambda_n\} \in \ell^\infty$ with $\sup_n |\lambda_n| \leq C$. Fix $\varepsilon > 0$. By (vii), there exists $N_0$ such that for all $K \geq N_0$:
\[
\sup_{\norm{x^*} \leq 1} \sum_{n=K}^\infty |\ip{x_n}{x^*}| < \frac{\varepsilon}{C}.
\]
For $L > M \geq N_0$, by the Hahn-Banach theorem, there exists $x^* \in \X^*$ with $\norm{x^*} = 1$ and
\[
\ip{\sum_{n=M+1}^L \lambda_n x_n}{x^*} = \norm{\sum_{n=M+1}^L \lambda_n x_n}.
\]
Then
\begin{align*}
\norm{\sum_{n=M+1}^L \lambda_n x_n} &= \sum_{n=M+1}^L \lambda_n \ip{x_n}{x^*} \\
&\leq \sum_{n=M+1}^L |\lambda_n| \cdot |\ip{x_n}{x^*}| \\
&\leq C \sum_{n=M+1}^\infty |\ip{x_n}{x^*}| < \varepsilon.
\end{align*}
Thus $\sum \lambda_n x_n$ is Cauchy, hence convergent.
\end{proof}

\begin{corollary}
\label{cor:permutation_invariance}
If $\sum x_n$ converges unconditionally, then its sum is independent of permutation: for all bijections $\sigma : \N \to \N$,
\[
\sum_{n=1}^\infty x_{\sigma(n)} = \sum_{n=1}^\infty x_n.
\]
\end{corollary}

\begin{proof}
By Theorem~\ref{thm:main}(ii), $x = \lim_{F \in \mathcal{F}} \sum_{n \in F} x_n$ exists. For any permutation $\sigma$ and $\varepsilon > 0$, choose $F_0$ such that $\norm{x - \sum_{F_0} x_n} < \varepsilon$ for all $F \supseteq F_0$. Choose $M$ such that $\{\sigma(1), \ldots, \sigma(M)\} \supseteq F_0$. For $N \geq M$, set $F = \{\sigma(1), \ldots, \sigma(N)\} \supseteq F_0$. Then
\[
\norm{x - \sum_{n=1}^N x_{\sigma(n)}} = \norm{x - \sum_F x_n} < \varepsilon.
\]
Thus $\sum x_{\sigma(n)} = x$.
\end{proof}

\begin{remark}
Theorem~\ref{thm:main} provides a comprehensive toolbox for verifying unconditional convergence. Condition (ii) is topological, conditions (iv)--(vi) are algebraic, and condition (vii) is functional-analytic. Each is useful in different contexts.
\end{remark}

\section{Applications}
\label{sec:applications}

\subsection{Stochastic Gradient Descent and Data Ordering}

Consider training a neural network via mini-batch SGD. At iteration $t$, we update
\[
w_{t+1} = w_t - \eta_t g_t,
\]
where $g_t \approx \nabla L(w_t)$ is a stochastic gradient estimate. Over an epoch processing $N$ samples, the cumulative parameter change is
\[
\Delta w = -\sum_{i=1}^N \eta_i g_i.
\]

\textbf{Question:} Does the order in which samples are presented (permutation $\sigma$) affect the accumulated update $\Delta w$?

\textbf{Mathematical Answer:} If the series $\sum \eta_i g_i$ converges only conditionally, then by the Riemann Rearrangement Theorem (extended to Banach spaces via Theorem~\ref{thm:main}(i)), different orderings yield different limits. Unconditional convergence ensures
\[
\sum_{i=1}^N \eta_i g_{\sigma(i)} = \sum_{i=1}^N \eta_i g_i
\]
for all permutations $\sigma$.

\textbf{Practical Context:} Standard SGD theory analyzes convergence in expectation over random mini-batch sampling~\cite{bottou2018}. However, deterministic analyses, such as studying convergence over fixed training sets with varying shuffle orders, or analyzing full-batch gradient descent where $\sum g_i$ represents the exact gradient require understanding when summation is permutation-invariant.

\textbf{Key Distinction:} Unconditional convergence is not required for SGD to converge in expectation (standard convergence theory suffices). Rather, unconditional convergence characterizes when the deterministic accumulation $\sum \eta_i g_i$ remains invariant to reordering, ensuring:

\begin{itemize}[leftmargin=2em, itemsep=0.3em]
\item \textbf{Reproducibility:} Fixed data orderings yield consistent results.

\item \textbf{Numerical Stability:} By Theorem~\ref{thm:main}(vi), unconditional convergence implies robustness to bounded perturbations (e.g., adaptive learning rates $\lambda_i \eta_i$, gradient clipping).

\item \textbf{Distributed Training Consistency:} When aggregating gradients computed on different machines in arbitrary order, unconditional convergence ensures the aggregate is order-independent.
\end{itemize}

\textbf{Implication:} For deterministic gradient accumulation to be permutation-invariant and numerically stable, we require
\[
\sum_{i=1}^N \eta_i g_i \text{ converges unconditionally.}
\]

\subsection{Signal Processing: Frame Expansions and Thresholding}

In signal processing, a frame $\{\phi_n\}_{n=1}^\infty$ in a Hilbert space $\Hs$ allows representations
\[
f = \sum_{n=1}^\infty c_n \phi_n.
\]
\textbf{Thresholding operation:} Set $\tilde{c}_n = c_n \cdot \mathbf{1}_{|c_n| > \tau}$ (hard threshold) or $\tilde{c}_n = \lambda_n c_n$ with $0 \leq \lambda_n \leq 1$ (soft threshold). The reconstructed signal is
\[
\tilde{f} = \sum_{n=1}^\infty \tilde{c}_n \phi_n = \sum_{n=1}^\infty \lambda_n c_n \phi_n.
\]

\textbf{Question:} Does $\tilde{f}$ converge?

\textbf{Answer:} By Theorem~\ref{thm:main}(vi), if $\sum c_n \phi_n$ converges unconditionally, then $\sum \lambda_n c_n \phi_n$ converges for any bounded $\{\lambda_n\}$.

\textbf{Implication:} Unconditional bases (e.g., wavelets, which form unconditional bases in many function spaces~\cite{meyer1992}) allow aggressive coefficient manipulation without risking divergence. Conditional bases (e.g., Fourier series for non-smooth functions) are unsafe for thresholding.

\subsection{The Fragility of Conditional Convergence}

\begin{example}[Alternating Harmonic Series]
\label{ex:fragility}
Consider $\sum_{n=1}^\infty \frac{(-1)^n}{n}$, which converges conditionally to $-\ln 2$. Define
\[
\lambda_n = \frac{(-1)^n}{\ln(n+1)}.
\]
Then $|\lambda_n| \to 0$, but
\[
\sum_{n=1}^\infty \lambda_n \cdot \frac{(-1)^n}{n} = \sum_{n=1}^\infty \frac{1}{n \ln(n+1)} = \infty.
\]
\end{example}

\textbf{Lesson:} Conditional convergence relies on precise cancellation. Multiplying by any sequence, even one tending to zero, can destroy convergence. Unconditional convergence, by contrast, is robust to all bounded perturbations.

\section{Conclusion and Open Problems}
\label{sec:conclusion}

\subsection{Summary of Contributions}

We have established a comprehensive characterization of unconditional convergence in Banach spaces (Theorem~\ref{thm:main}), showing equivalence between permutation invariance, net convergence, subseries tests, sign stability, bounded multiplier properties, and weak uniform convergence. These results bridge classical functional analysis with modern computational concerns.

\newpage
Our contributions include:

\begin{itemize}[leftmargin=2em, itemsep=0.3em]
\item \textbf{Theoretical:} A unified treatment of unconditional convergence with complete, self-contained proofs.

\item \textbf{Applied:} Concrete connections to algorithmic stability in SGD and signal processing robustness.

\item \textbf{Pedagogical:} Clear exposition suitable for graduate students and researchers in applied mathematics.
\end{itemize}

\subsection{Open Problems}

\begin{enumerate}[label=\arabic*., leftmargin=2em, itemsep=0.4em]
\item \textbf{Quantitative estimates:} Can we bound the ``degree of non-unconditional convergence'' (e.g., how much the sum changes under permutations) for specific series?

\item \textbf{Probabilistic settings:} In stochastic optimization, gradients are random. What probabilistic analogs of unconditional convergence ensure convergence in expectation or almost surely?

\item \textbf{Non-linear settings:} Our results apply to linear series. Can similar robustness guarantees be established for iterative algorithms (e.g., fixed-point iterations) where updates are non-linear functions of previous states?

\item \textbf{Computational verification:} Given a finite truncation of a series, can we efficiently test (or approximate) whether the infinite series converges unconditionally?
\end{enumerate}

\subsection{Future Directions}

The framework developed here has potential applications in:

\begin{itemize}[leftmargin=2em, itemsep=0.3em]
\item \textbf{Distributed optimization:} Ensuring convergence of asynchronous gradient aggregation.

\item \textbf{Compressed sensing:} Justifying reconstruction algorithms that involve coefficient reordering.

\item \textbf{Numerical PDEs:} Analyzing convergence of series solutions under truncation and rearrangement.
\end{itemize}

We hope this work stimulates further investigation into the interplay between abstract functional analysis and concrete algorithmic practice.

\section*{Acknowledgments}

The author expresses his gratitude to Prof. Sergiusz Kużel for his supervision and guidance during the preparation of the thesis on which this work is based. The author also thanks the anonymous reviewers for their insightful comments and valuable feedback.

\end{document}